\documentclass[11pt]{article}
\usepackage{amsmath,amssymb}
\usepackage{latexsym}
\usepackage{graphicx}
\usepackage{algorithm}
\usepackage{amsthm}
\usepackage{authblk}

\newtheorem{definition}{Definition}
\newtheorem{theorem}[definition]{Theorem}
\newtheorem{corollary}[definition]{Corollary}

\newtheorem{proposition}[definition]{Proposition}

\newtheorem{example}[definition]{Example}

\newcommand{\avec}{{\bf a}}

\newcommand{\xvec}{{\bf x}}

\newcommand{\yvec}{{\bf y}}
\newcommand{\zvec}{{\bf z}}
\newcommand{\fvec}{{\bf f}}
\newcommand{\gvec}{{\bf g}}

\begin{document}

\title{On the Size and Width of the Decoder of a Boolean Threshold Autoencoder}
\author[1]{Tatsuya Akutsu}
\author[2]{Avraham A. Melkman}
\affil[1]{Bioinformatics Center, Institute for Chemical Research, Kyoto University}
\affil[2]{Department of Computer Science,
Ben-Gurion University of the Negev}

\maketitle

\begin{abstract}%
In this paper, we study the size and width of autoencoders consisting
of Boolean threshold functions,
where an autoencoder is a layered neural network
whose structure can be viewed as consisting of an \emph{encoder},
which compresses an input vector to a lower dimensional vector,
and a \emph{decoder} which transforms
the low-dimensional vector back to the original input vector exactly
(or approximately).
We focus on the decoder part, and
show that $\Omega(\sqrt{Dn/d})$ and $O(\sqrt{Dn})$ nodes are required
to transform $n$ vectors in $d$-dimensional binary space
to $D$-dimensional binary space.
We also show that the width can be reduced if we allow small errors,
where the error is defined as the average of the Hamming distance between
each vector input to the encoder part and the resulting vector output
by the decoder.

{\bf Keywords:}
Neural networks, Boolean functions, threshold functions, autoencoders.
\end{abstract}

\section{Introduction}
\label{sec:intro}
Extensive studies have been done on artificial neural networks
not only in artificial intelligence and machine learning but also
in theoretical computer science \cite{anthony01,sima93,siu95}.
Among various models of neural networks,
\emph{autoencoders} 
%>AM
%recently attract much attention due to their power
have recently attracted much attention due to their ability
%<AM
to generate new data, and
have been applied to various areas including
image processing \cite{doersch16,tschannen18},
natural language processing \cite{tschannen18},
and drug discovery \cite{gomez18}.
An autoencoder is a layered neural network consisting of two parts,
an \emph{encoder} and a \emph{decoder},
where the former transforms an input vector to
a low-dimensional vector and the latter transforms the low-dimensional
vector to an output vector which should be
the same as or similar to the input vector
\cite{ackley85,baldi89,hinton06,baldi12}.
Therefore, an autoencoder maps input data to a low-dimensional
representation space.
Such a mapping is obtained via unsupervised learning that minimizes
the difference between input and output data by adjusting weights
(and some other parameters).

Although autoencoders perform dimensionality reduction,
a kind of data compression,
how data are compressed via autoencoders is still unclear.
In particular,
the quantitative relationship between the compressive power and
the numbers of nodes and layers in autoencoders had been unclear.
In order to study the relationship, Melkman et al. analyzed
the relations between
the architecture (e.g., the numbers of nodes and layers) of networks
and their compression ratios \cite{melkman21},
using a layered \emph{Boolean threshold network},
which is a discrete model of neural networks.
A Boolean threshold network is equivalent to a \emph{threshold circuit}
\cite{anthony01,sima93,siu95}
in which
each node takes on values that are either 1 (active) or 0 (inactive)
and the activation rule for each node is given by a Boolean threshold function.
%>AM
%They showed several architectures of autoencoders to transform $n$ input vectors
%in $D$-dimensional binary space
In \cite{melkman21} several architectures of autoencoders were presented
that map $n$ $D$-dimensional binary input vectors
%<AM
into $d$-dimensional binary space and then recover the original input vectors.
In particular, they showed the following architectures for $d = \lceil \log n \rceil$ (or, almost equivalently, $d = 2 \lceil \log \sqrt{n} \rceil$):
\begin{itemize}
\item a 4-layers encoder with $O(\sqrt{n}+D)$ nodes;
\item a 5-layers autoencoder with $D/n/d/n/D$ architecture,
where each parameter means
the number of nodes in the corresponding layer;
\item a 7-layers autoencoder with $O(D \sqrt{n})$ nodes;
\item a decoder with depth $n+1$ and width $O(D)$,
where the width is the maximum number of nodes per layer (except for the input and
output layers) and the depth is the number of layers minus one.
\end{itemize}
However, it was unclear whether or not these results are optimal (or near optimal).

In this paper, we focus on the decoder part because the decoders (in the
%>AM
%above)
above-quoted results)
%<AM
use $O(D \sqrt{n})$, $O(n+D)$, or more nodes whereas the encoder uses
only
$O(\sqrt{n}+D)$ nodes.
We show an $\Omega(\sqrt{Dn/d})$ lower bound on the size of
the perfect decoder, where a decoder is called \emph{perfect} if
it can always
%>AM
%and exactly recover the original input vectors.
recover the original input vectors exactly.
%<AM
To our knowledge,
this is the first lower bound result on the size (i.e., the number of nodes)
of the autoencoder.
As a positive result,
we show that there exists a perfect decoder with
width $\max(\lceil n/B \rceil+B,BD)$ and a constant depth,
where $B$ is an arbitrary integer larger than 1.
By letting $B = \lceil \sqrt{n/D} \rceil$,
we obtain an $O(\sqrt{Dn})$ upper bound
on the size and width of the decoder, which is relatively close
to the lower bound and improves the previous $O(D \sqrt{n})$ upper
bound.
In order to
%>AM
%further improve the width of the decoder,
%we allow errors in the output vector, where the error is defined
%as the average Hamming distance between the input and output vectors.
construct decoders that have even smaller width we permit them
to output vectors that are in error,
where the error is defined as the average of the Hamming distance between each vector
input to the encoder part and the resulting vector output by the decoder.
%<AM
We show that there exists a decoder with
width $\max(\lceil n/B \rceil +B,(B-1)D+1)$ and a constant depth whose
error is at most $D({\frac {1}{B 2^B}}+{\frac 1 n})$.

\section{Preliminaries}
\label{sec:pre}

A function $f$: $\{0,1\}^h \rightarrow \{0,1\}$ is called
a \emph{Boolean threshold function} if it is represented as
\begin{eqnarray*}
f(\xvec) = \left\{
\begin{array}{ll}
1, & \avec \cdot \xvec \geq \theta,\\
0, & \mbox{otherwise,}
\end{array}
\right.
\end{eqnarray*}
for some $(\avec,\theta)$,
where $\avec$ is an $h$-dimensional integer vector and
$\theta$ is an integer.
We will also denote the same function as
$[\avec \cdot \xvec \geq \theta]$.
An acyclic network is called a \emph{Boolean threshold network} (BTN)
if all activation functions in the network are Boolean threshold
functions.

In this paper, we only consider layered BTNs in which
nodes are divided into $L$-layers and
each node in the $i$-th layer has inputs only from nodes in
the ($i-1$)th layer ($i=1,\ldots,L-1$).
Then, the states of nodes in the $i$-th layer can be represented as
a $W_i$-dimensional binary vector where $W_i$ is the number of nodes
in the $i$-th layer and is called the \emph{width} of the layer.
A layered BTN is represented as
$\yvec = \fvec^{(L-1)}(\fvec^{(L-2)}( \cdots \fvec^{(1)}(\xvec) \cdots ))$,
where
$\xvec$ and $\yvec$ are the input and output vectors, respectively.
and $\fvec^{(i)}$ is a list of activation functions for the ($i+1$)th layer.
The $0$-th and ($L-1$)-th layers are called the \emph{input and output layers},
respectively,
and the corresponding nodes are called \emph{input and output nodes}, respectively.
The \emph{size}, \emph{depth}, \emph{width} of a layered BTN are defined as
the number of nodes, the number of layers minus one,
and the maximum number of nodes in a layer (excepting the input
and output layers), respectively.
When we consider autoencoders,
one layer ($k$-th layer where $k \in \{1,\ldots,L-2\}$)
is specified as the \emph{middle layer}, and the nodes in this layer
are called the \emph{middle nodes} (see also Fig.~\ref{fig:5-layer}).
Then, the \emph{middle vector} $\zvec$, \emph{encoder} $\fvec$, 
and \emph{decoder} $\gvec$ are defined by
\begin{eqnarray*}
\zvec & = & \fvec^{(k-1)}(\fvec^{(k-2)}(\cdots \fvec^{(1)}(\xvec) \cdots )) = \fvec(\xvec),\\
\yvec & = & \fvec^{(L-1)}(\fvec^{(L-2)}(\cdots \fvec^{(k)}(\zvec) \cdots )) = \gvec(\zvec).
\end{eqnarray*}
Since the input vectors should be recovered from the middle vectors at the output layer,
we assume that $\xvec$ and $\yvec$ are $D$-dimensional binary vectors
and $\zvec$ is a $d$-dimensional binary vector with $d \leq D$.
A list of functions is also referred to as a \emph{mapping}.
The $i$-th element of a vector $\xvec$ will be denoted
by $x_i$, which is also used to denote the node corresponding to this element.
Similarly, for each mapping $\fvec$,
$f_i$ denotes the $i$-th function.

\begin{figure}[ht]
\centering
\includegraphics[width=12cm]{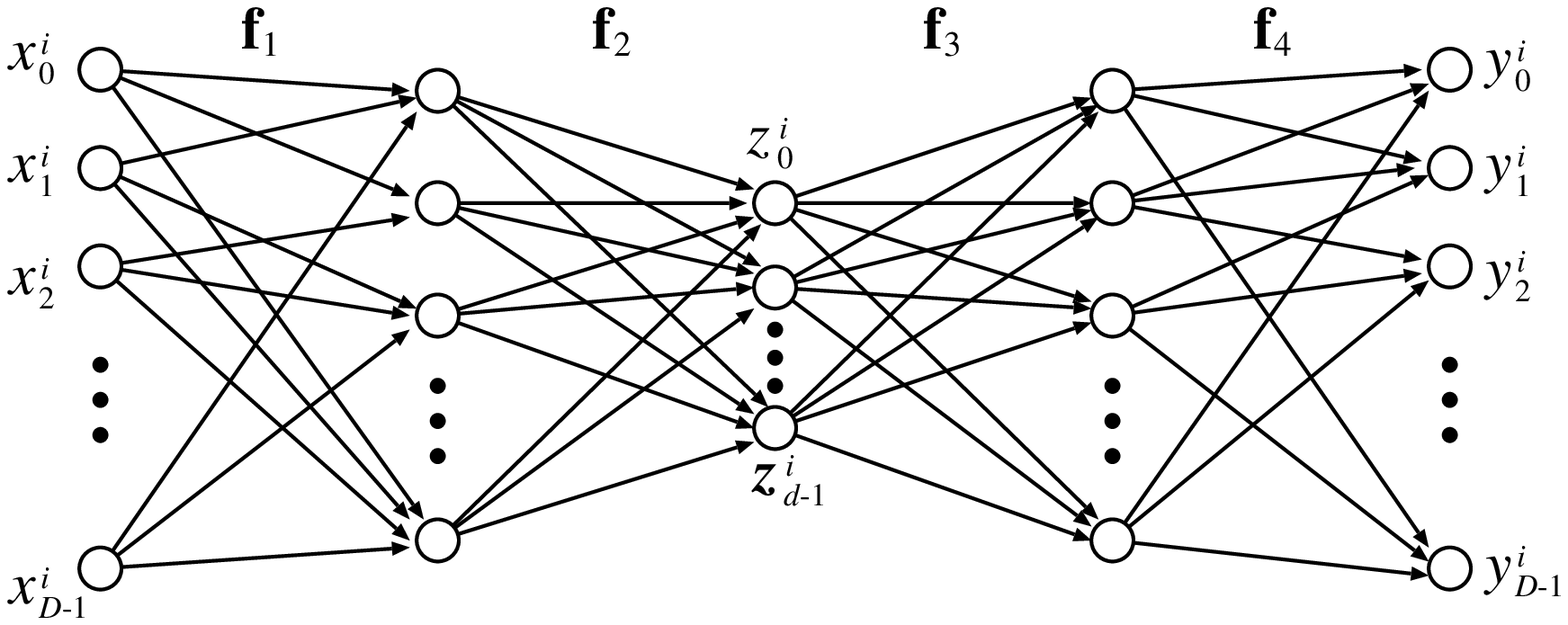}
\caption{Architecture of a five-layer autoencoder,
where $\zvec = \fvec(\xvec)=\fvec_2(\fvec_1(\xvec))$ and
$\yvec = \gvec(\zvec) = \fvec_4(\fvec_3(\zvec))$.}
\label{fig:5-layer}
\end{figure}

Let
$X_n = \{\xvec^0,\ldots,\xvec^{n-1}\}$
be a set of $n$ $D$-dimensional binary input vectors
that are all different.
We define perfect encoder, decoder, autoencoder as follows \cite{melkman21}.

\begin{definition}
A mapping $\fvec$: $\{0,1\}^D \rightarrow \{0,1\}^d$
is called a \emph{perfect encoder} for $X_n$
if $\fvec(\xvec^i) \neq$ $\fvec(\xvec^j)$ holds for all $i \neq j$.
\end{definition}

\begin{definition}
A pair of mappings $(\fvec,\gvec)$ with
$\fvec$: $\{0,1\}^D \rightarrow \{0,1\}^d$
and
$\gvec$: $\{0,1\}^d \rightarrow \{0,1\}^D$
is called a \emph{perfect autoencoder} if
$\gvec(\fvec(\xvec^i))=\xvec^i$ holds for all $\xvec^i \in X_n$.
Furthermore, such $\gvec$ is called a \emph{perfect decoder}.
\end{definition}

Note that a perfect decoder exists only if there exists a perfect
autoencoder.
Furthermore, it is easily seen from the definitions
that if $(\fvec,\gvec)$ is a perfect autoencoder,
$\fvec$ is a perfect encoder.

\begin{example}
Let $X_4 = \{\xvec^0,\xvec^1,\xvec^2,\xvec^3\}$ where
$\xvec^0=(0,0,0)$,
$\xvec^1=(1,0,0)$,
$\xvec^2=(1,0,1)$,
$\xvec^3=(1,1,1)$.
Let $D=3$ and $d=2$.
Define $\zvec=\fvec(\xvec)$ and $\yvec=\gvec(\zvec)$ by
        \begin{eqnarray*}
                f_0(\xvec) & = & [x_0 + x_1 - x_2 \geq 1],~~~
                f_1(\xvec) ~=~ [x_2 \geq 1],\\
                g_0(\zvec) & = & [z_0 + z_1 \geq 1],~~~
                g_1(\zvec) ~=~ [z_0 + z_1 \geq 2],~~~
                g_2(\zvec) ~=~ [z_1 \geq 1].
        \end{eqnarray*}
This pair of mappings has the following truth table, which shows it to be
a perfect autoencoder.

\smallskip

\begin{center}
        \begin{tabular}{|lll|ll|lll|}
                        \hline
                        $x_0$ & $x_1$ & $x_2$ & $z_0$ & $z_1$ & $y_0$ & $y_1$ & $y_2$\\
                        \hline
                        0 & 0 & 0 & 0 & 0 & 0 & 0 & 0\\
                        1 & 0 & 0 & 1 & 0 & 1 & 0 & 0\\
                        1 & 0 & 1 & 0 & 1 & 1 & 0 & 1\\
                        1 & 1 & 1 & 1 & 1 & 1 & 1 & 1\\
                        \hline
        \end{tabular}
\end{center}
\end{example}

\section{A Lower Bound on the Size of the Decoder}
%% TA>
\label{sec:lower-bound}
%% <TA

In this section we derive a lower bound on the number of threshold units
(i.e., the number of nodes except those in the input layer)
in the decoder BTN of a Boolean threshold autoencoder
which perfectly encodes every possible set of $n$ $D$-dimensional vectors, and which has a middle layer whose size is less than one third the size of the output layer. 
Note that $n\leq 2^d$, since $n$ vectors are perfectly encoded. 

Denote the number of threshold units of the decoder-BTN by $N$, with the 
$i$-th  unit having $d_i$ inputs,
%% TA21dec>
where $N$ includes the number of output nodes but does not include
the number of middle nodes.
%% <TA21dec

\begin{theorem}
%% TA21dec>
%Suppose a BTN autoencoder with a middle layer of size $d$ perfectly autoencodes 
%every set of $n$ $D$-dimensional vectors.
%Then its decoder BTN has 
%at least $\sqrt{\frac{D-1}{3d}}\sqrt{n}$ nodes
%provided
%$d \leq {\frac {D-1}{3}}$.
There exists a set $X_n$ for which there is no perfect autoencoder
with a middle layer of size $d$ and a decoder of size less than 
$\sqrt{\frac{D-1}{3d}}\sqrt{n}$.
%% <TA21dec
\label{thm:lower-bound}
\end{theorem}
\begin{proof}
The idea is to calculate an upper bound on the total number of different sets of $n$ vectors 
that can be generated at the output layer of the autoencoder from 
all possible
sets of $d$-dimensional vectors, input to the middle layer, by all possible different
decoder BTNs. Let's denote that number $PD$ (the in-Principle Decoded sets). 
Since the autoencoder is perfect, $PD$ is not smaller than the number of different sets of $D$-dimensional output (=input) vectors of the autoencoder, each of size $n$, which 
is $\binom{2^D}{n}
%\geq {\frac {(2^D-n)^n} {n^n}}
$.

%% AM21.12.21>
%Let us now bound $PD$. 
Next we derive an upper bound on $PD$.
%% <AM21.12.21
Each threshold unit computes a Boolean function of 
the $d$ inputs (in the middle-layer). 
In the course of proving Theorem 7.4 of \cite{anthony01}, 
%% AM21.12.21>
%that gives 
which provides
%% <AM21.12.21
a lower bound on the size of the universal networks,
it is shown that the number of different functions representable
by a threshold unit with in-degree $d_i$
in an acyclic BTN which has $d$ input nodes is at most
$2^{d d_i + 1}$.
Therefore, an upper bound on
the total number of functions from $\{0,1\}^d$ to $\{0,1\}^D$
computable by a decoder BTN with $N$ nodes  is
\begin{equation*}
\prod_{i=1}^N 2^{d d_i+1}  =  2^{N+d \sum_{i=0}^{N-1} d_i}
\leq  2^{N + dN^2}.
\end{equation*}
Since the number of all possible ordered sets of $n$ $d$-dimensional vectors
%% TA21ec >
% vectors
%% < TA21ec 
in the middle layer does not exceed $2^d!/(2^d-n)! \leq (2^d)^n$ we conclude that

\begin{eqnarray*}
%\left[ (2^d)^n \right] \cdot \left[ 2^{N + d N^2} \right] & \geq & \binom{2^D}{n}
%\geq {\frac {(2^D-n)^n} {n^n}}.
{\frac {(2^D-n)^n} {n^n}}\leq \binom{2^D}{n} \leq PD \leq 
\left[ (2^d)^n \right] \cdot \left[ 2^{N + d N^2} \right]
\end{eqnarray*}
Taking base 2 logarithms of both sides shows that the parameters of the decoder satisfy
$n d + N + d N^2  \geq  n \log(2^D-n) - n \log n$.
Noting that $n\leq 2^d < 2^{D-1}$
yields $N + d N^2  >  n (D-1) - 2 n d$.
Applying the precondition
$d \leq {\frac {D-1}{3}}$ 
results in the inequality 
\begin{eqnarray*}
N + d N^2 > {\frac n 3} (D-1) ,
\end{eqnarray*}
from which the theorem follows.
\end{proof}

%% TA21dec>
%% AM21.12.21>
% There are some remarks on this theorem.
{\bf Remarks:}
%% <AM21.12.21
\begin{itemize}
\item This lower bound is meaningful only if 
$\sqrt{\frac{D-1}{3d}}\sqrt{n} \geq D$ because the decoder contains
$D$ output nodes.
\item 
%% AM21.12.21>
%Since $d_i=N$ is used in the proof, this lower bound 
In the proof $d_i$ was replaced by its upper bound $N$.
Therefore, the lower bound stated in the Theorem 
%% <AM21.12.21
holds for any acyclic decoder.
\item 
%% AM21.12.21>
%This result gives an $\Omega(\sqrt{Dn/d})$ lower bound
%on the width for layered BTN decoders of a constant depth.
The stated lower bound becomes an $\Omega(\sqrt{Dn/d})$ lower bound
on the widths of layered BTN decoders of constant depth.
%% <AM21.12.21
\end{itemize}
% Note that this lower bound holds also for non-layered BTNs.
% Note also that this result gives an $\Omega(\sqrt{Dn/d})$ lower bound
% on the width for layered BTN decoders of a constant depth.
%% <21decTA
%
%
%
%% TA>
\section{Decoder with Width $\max(\lceil n/B \rceil +B,BD)$}
\label{sec:perfect-decoder}

In this section we present the architecture of a decoder BTN 
that is able to decode perfectly any set of $n$ $D$-dimensional vectors
that has been encoded into a layer of size $d$
(so that the encoder and decoder together perfectly autoencode every such set),
where $d = \lceil \log n \rceil$.
%% < TA

We will use $\beta_0,\ldots,\beta_{d-1}$ to refer to the middle nodes
as well as to their values, with $\beta$ denoting the vector $(\beta_0,\ldots,\beta_{d-1})$.
$Int(\beta)$ denotes the integer encoded by $\beta$, i.e. 
%% TA>
$Int(\beta)=\sum_{i=0}^{d-1}\beta_i 2^{d-i-1}$.
%% <TA

%
% \subsection{Decoder with Width $\max({\frac n B}+B,BD)$}

Let $B$ be an integer greater than 1.
For ease of exposition we assume that $n$ can be divided by $B$ and
use $n_B$ to denote $n/B$.
Otherwise we can let $n_B = \lceil n/B \rceil$.
In this section we detail the architecture of a $d/{\frac n B}+B / BD / D$ 
decoder (with $B=2^K$ for some integer $K$), 
see Figure \ref{fig:width1}. In addition to $\beta$ it has the following nodes:
%<AM

\begin{itemize}
\item $\gamma_0,\gamma_1,\ldots,\gamma_{n_B-1},\gamma_{n_B},\gamma_{n_B+1},\ldots,\gamma_{n_B+B-1}$,
\item $y_{j,b}$ for $j=0,\ldots,D-1$ and $b=0,\ldots,B-1$,
\item $y_{j}$ for $j=0,\ldots,D-1$, the output nodes.
\end{itemize}

The activation function of the $\gamma_i$ are:
\begin{eqnarray*}
\gamma_i & = & [\lfloor {\frac {Int(\beta)} {B} } \rfloor = i],~~~\mbox{for $i=0,\ldots,n_B-1$},\\
\gamma_{n_B+h} & = & [Int(\beta)=h~(\mbox{mod $B$})],~~~\mbox{for $h=0,\ldots,B-1$},
\end{eqnarray*}
where $[Int(\beta)=k]$ takes value 1
if $Int(\beta)=k$, and value 0 otherwise.
Note that division and modulo operations can be done by
threshold circuits of width $O(poly(\log n))$ and depth 5 \cite{siu93}
(recall that $\beta$ is a vector of dimension $O(\log n)$),
%%AM.21.06.16>
%and thus are negligible
% (because our networks have width $O(n)$).
a negligible increase to the $O(\sqrt{Dn})$ width of our network.
%<AM
Furthermore, if $B = 2^K$ for some integer $K$,
we do not need such circuits
%% TA>
because we can simply use the first $d-K$ bits of $\beta$
to represent $Int(\beta)/B$ and 
the remaining $K$ bits to represent $Int(\beta) \mbox{ mod $B$}$.
%% <TA
Note also that $[Int(\beta)=k]$ can be calculated by
using a single unit with an activation function
$\avec \cdot \beta \geq \theta$ such that
\begin{eqnarray*}
a_i & = & \left\{
\begin{array}{ll}
1, & \mbox{if $k_i = 1$},\\
-1, & \mbox{otherwise},
\end{array}
\right.\\
\theta & = & \sum_i k_i,
\end{eqnarray*}
where $k_0 k_1 \cdots k_{d-1}$ is the binary representation of
an integer $k$.

%%AM.21.06.16>
%The activation function of each $y_{j,b}$ is defined by
The activation function of $y_{j,b}$ is 
%<AM
$$
y_{j,b} = [\sum_{i=0}^{n_B+B-1} w_{i,j,b} \gamma_i \geq 2],
$$
where $w_{i,j,b}$ is given by
\begin{eqnarray*}
w_{h,j,b} & = & x_j^{Bh+b}~~~\mbox{for $h=0,\ldots,n_B-1$ and for $b=0,\ldots,B-1$},\\
w_{n_B+h,j,b} & = & [h=b]~~~\mbox{for $h=0,\ldots,B-1$ and for $b=0,\ldots,B-1$}.
\end{eqnarray*}

%%AM.21.06.16>
%It is seen that $y_{j,b}=1$ if and only if $x^{Bh+b}_j=1$.
Numbering the $n$ input/output vectors so that each
%% TA>
$\xvec^k$
%% <TA
is mapped to $\beta$ with 
$Int(\beta)=k$, it is not difficult to verify that when
%% TA>
$\xvec^k$
%% <TA
 is input to the autoencoder, with $k=B\ell+r,\ 0\leq r \leq B-1$, then $y_{j,b}=[x^{B\ell+b}_j+[b=r]\geq 2]$.
Hence $y_{j,b}=1$ if and only if $b=r$ and $x^k_j=1$.
%<AM

The activation function of $y_{j}$ is 
$$
y_j = [\sum_{b=0}^{B-1} y_{j,b} \geq 1],
$$
%%AM.21.06.16>
which ensures that on input
%% TA>
$\xvec^k$
the value of $y_j$ is $x^k_j$,
i.e. the decoder is perfect.
%% <TA

%% TA21dec>
Although we assumed that $\xvec^k$ is mapped to $\beta$ with
$Int(\beta)=k$,
this assumption does not pose any restriction on an encoder
because arbitrary permutations of input vectors can be considered.
Therefore, we have:
%% <TA21dec

\begin{theorem}
%% TA>
\label{thm:perfect-decoder}
%% <TA
%% TA21dec>
For any perfect encoder that maps $X_n$
% 21dec >
one-to-one
to $d$-dimensional binary vectors,
with $d=\lceil \log n \rceil$ 
%% AM21.12.21>
%for a sufficiently large $n$,
and $n$ sufficiently large,
%% <AM21.12.21
there exists a perfect decoder of a constant depth and width at most
$\max(\lceil n/B \rceil+B,BD)$ 
where $B$ is any integer such that $1 < B \leq n$.
% There exists a perfect decoder of a constant depth and width at most
% $\max(\lceil n/B \rceil+B,BD)$ for $n \gg (\log n)^k$ for some constant $k$,
% where $B$ is any integer such that $1 < B \leq n$.
%% <TA21dec
\end{theorem}

Note that if $B = 2^K$ for some integer $K$,
%% TA21dec>
%we do not need a condition of $n \gg (\log n)^k$).
we do not need a condition of ``sufficiently large $n$''
(precisely, $n \gg (\log n)^k$ for some constant $k$).
This condition is needed only for calculating $\lceil Int(\beta)/B \rceil$ and
($Int(\beta)\mod B$), the details of which are not relevant to this paper.
%% <TA21dec

\begin{figure}[ht]
\centering
\includegraphics[width=6cm]{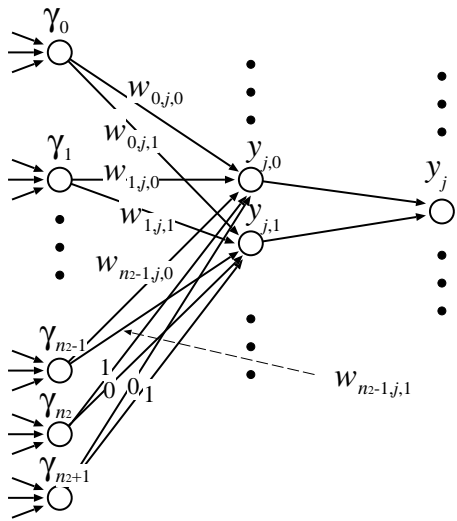}
\caption{Network structure of a decoder with width $\max(\lceil n/B \rceil+B,BD)$
for $B=2$.}
\label{fig:width1}
\end{figure}

%
%%AM.21.06.16>
%Here we consider the case of ${\frac n B} = BD$.
%Then, we have $B = \sqrt{{\frac n D}}$.
%By letting $B = \lceil \sqrt{{\frac n D}} \rceil$,
%we have the following corollary.
Upon choosing $B = \lceil \sqrt{{\frac n D}} \rceil$ in
%% TA>
Theorem~\ref{thm:perfect-decoder}
%% <TA
 we obtain the following corollary.
%<AM
\begin{corollary}
Suppose $n > D$.
%% TA21dec>
Then, for any perfect encoder that maps $X_n$
% 21dec >
one-to-one
to $d$-dimensional binary vectors
with $d = \lceil \log n \rceil$,
there exists a constant depth perfect decoder width $O(\sqrt{D n})$ nodes.
% Then, there exists a constant depth perfect decoder with $O(\sqrt{D n})$ nodes.
%% <TA21dec
\end{corollary}

Interestingly, this corollary improves 
the $O(D\sqrt{n})$ size of a perfect decoder given in \cite{melkman21},
while using a simpler architecture.
Furthermore, the order of this upper bound is close to that of
the lower bound given in Theorem~\ref{thm:lower-bound}.

\section{Approximate Decoders}
\label{sec:approx}

In Section~\ref{sec:perfect-decoder},
we showed a construction of a perfect decoder
of width $\max(\lceil n/B \rceil+B,BD)$.
It seems not easy to reduce the width using a constant depth.
Here we show that we can reduce the width 
if small errors between the original input vectors
(to the encoder) and the output vectors are permissible.
%% TA>
This assumption is reasonable because the input and output vectors
are not necessarily the same in practice.
%% <TA

In order to measure the error, we employ the Hamming distance
between the original
vector $\xvec$ and the decoded vector $\yvec$,
and denote it by $dist(\xvec,\yvec)$.
The average Hamming distance $dist(X,Y)$ between
the sets of vectors
$X=(\xvec^0,\ldots,\xvec^{n-1})$ and
$Y=(\yvec^0,\ldots,\yvec^{n-1})$
is defined as
$$
dist(X,Y) = {\frac 1 n} \sum_{i=0}^{n-1} dist(\xvec^i,\yvec^i).
$$

\subsection{Case of $B=3$}
\label{sec:approx-b-3}

As in Section~\ref{sec:perfect-decoder},
we divide the input vectors into $B$ sets.
In this subsection, we consider the case of $B=3$ for ease of
presentation.
We will extend it to a general $B$ in the next subsection.

We assume w.l.o.g. that ${\frac n 3}$ is a positive integer and
define $n_3 = {\frac n 3}$.
We construct $\gamma_i$ nodes $(i=0,\ldots,n_3+2)$,
$y_{i,0}$ nodes $(i=0,\ldots,D-1$), and $y_{i,1}$ nodes $(i=0,\ldots,D-1)$ 
as follows (see Fig.~\ref{fig:approx-decode-b3}).

Recall that $Int(\beta)$ denotes the integer coded by a binary vector $\beta$.
Then, the activation function of $\gamma_i$ is defined by
\begin{eqnarray*}
\gamma_i & = & [ \lfloor {\frac {Int(\beta)}{3}} \rfloor = i],~~~~
\mbox{for $i=0,\ldots,n_3-1$},\\
\gamma_{n_3} & = & [ Int(\beta) = 0 ~(\mbox{mod}~3)],\\
\gamma_{n_3+1} & = & [ Int(\beta) = 1 ~(\mbox{mod}~3)],\\
\gamma_{n_3+2} & = & [ Int(\beta) = 2 ~(\mbox{mod}~3)].
\end{eqnarray*}
The activation function for $y_{j,h}, j=0,\ldots, D-1, h=0,1$,
is
$$
y_{j,h}=[\sum^{n_3+2}_{i=0} w^j_{i,h} \gamma_i \geq 0].
$$

The values of $w^j_{i,h}$s ($i=0,\ldots,n_3-1$, $h=0,1$) depend on $x^{3i}_j x^{3i+1}_j x^{3i+2}_j$
and are determined as in Table~\ref{tbl:definition-w},
where each weight $w^j_{i,h}$ is shown by its binary representation
(e.g., $w^j_{i,0} = 4$ and $w^j_{i,1}=1$
if $x^{3i}_j x^{3i+1}_j x^{3i+2}_j = 010$).
Note that the corresponding values of the $y_{j,h}$s and $z_j$s are also shown
in the table.
The values of $w^j_{i,h}$s ($i=n_3,n_3+1,n_3+2$, $h=0,1$) are defined by
\begin{eqnarray*}
w^j_{{n_3},0} = -2, & w^j_{{n_3+1},0}=-4,~~~w^j_{{n_3+2},0}=0,\\
w^j_{{n_3},1} = -2, & w^j_{{n_3+1},1}=0,~~~~w^j_{{n_3+2},1}=-4 ,
\end{eqnarray*}
where these $w^j_{i,h}$s do not depend on $j$.

Then, for each $j=0,\ldots,D-1$, we construct a node $z_j$ with
the activation function defined by
$$
z_j = \left[ y_{j,0} + y_{j,1} \geq 2 \right].
$$

\begin{table}[h]
\caption{Definitions of weight $w^j_{i,h}$s and values of relevant variables.}
\label{tbl:definition-w}
\begin{center}
\begin{tabular}{c||c|c||ccc}
\hline
$x^{3i}_j x^{3i+1}_j x^{3i+2}_j$ & $w^j_{i,0}$ & $w^j_{i,1}$ &
$y^{3i}_{j,0} y^{3i+1}_{j,0} y^{3i+2}_{j,0}$ &
$y^{3i}_{j,1} y^{3i+1}_{j,1} y^{3i+2}_{j,1}$ &
$z^{3i}_j z^{3i+1}_j z^{3i+2}_j$ \\
\hline
000 & 000 & 000 & 001 & 010 & 000 \\
001 & 001 & 100 & 001 & 111 & 001 \\
010 & 100 & 001 & 111 & 010 & 010 \\
011 & 101 & 101 & 111 & 111 & 111 \\
100 & 010 & 010 & 101 & 110 & 100 \\
101 & 011 & 110 & 101 & 111 & 101 \\
110 & 110 & 011 & 111 & 110 & 110 \\
111 & 111 & 111 & 111 & 111 & 111 \\
\hline
\end{tabular}
\end{center}
\end{table}
We use $w_h(\ldots)$ ($h=0,1$) to denote the functions
given in the above.
For example, $w_0(010)=100_{(2)}=4$, $w_1(010)=001_{(2)}=1$,
$w_0(110)=110_{(2)}=6$, and $w_1(110)=011_{(2)}=3$,
where $x_{(2)}$ means that $x$ is a binary representation of an integer.

\begin{figure}[ht]
\centering
\includegraphics[width=10cm]{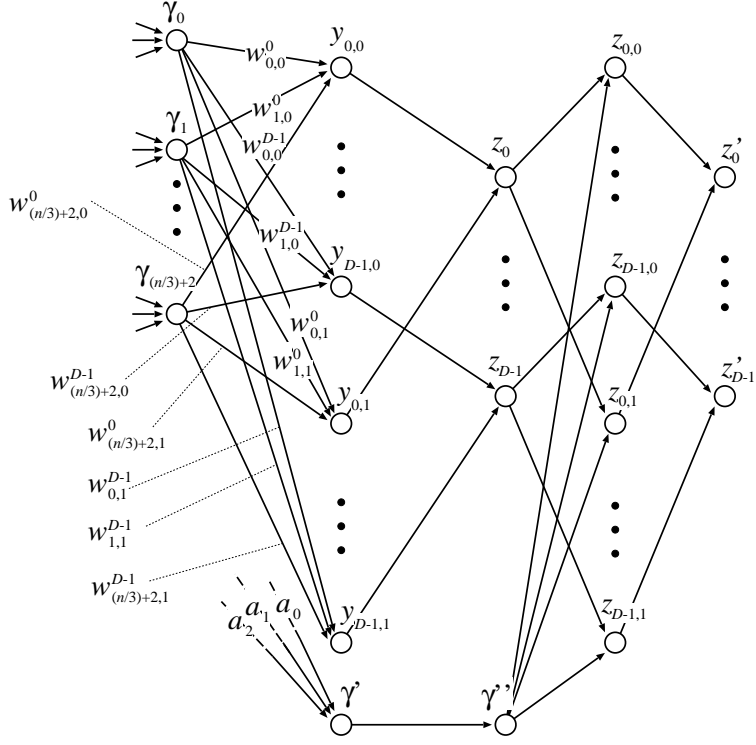}
\caption{Network structure for approximate decoding with $B=3$.}
\label{fig:approx-decode-b3}
\end{figure}

It is seen from this table that an error occurs only when 
$x^{3i}_j x^{3i+1}_j x^{3i+2}_j = 011$.
Therefore, if all bit patterns are distributed uniformly at random,
the expected Hamming distance between $\xvec^i$ and $\yvec^i$
is ${\frac {D}{3 \cdot 2^3}} = {\frac {D}{24}}$.
We can modify the network so that 
the average Hamming distance between  $\xvec^i$ and $\yvec^i$
is always no more than ${\frac {D}{24}}$.

This can be done as follows.
Let 
$
\#_{b_0 b_1 b_2}
= \sum_{j=0}^{D-1} |\{ i | x^{3i}_j x^{3i+1}_j x^{3i+2}_j = b_0 b_1 b_2 \}|$
where $b_0 b_1 b_2 \in \{0,1\}^3 $.
Let $c_0 c_1 c_2 = \mbox{argmin}_{b_0 b_1 b_2} \#_{b_0 b_1 b_2}$,
where the tie can be broken in any way.
Then, we let $a_0 a_1 a_2 = (0 \oplus c_0) (1 \oplus c_1) (1 \oplus c_2)$.
For example. $a_0 a_1 a_2 = 011$ when $c_0 c_1 c_2 = 000$.

Here we define $\chi_{a_0 a_1 a_2}(x_0 x_1 x_2)$ by
$$
\chi_{a_0 a_1 a_2}(x_0 x_1 x_2) = (x_0 \oplus a_0) (x_1 \oplus a_1) (x_2 \oplus a_2).
$$
For example, 
$\chi_{011}(000)=011$, $\chi_{011}(001)=010$, and $\chi_{011}(011)=000$
when $a_0 a_1 a_2 = 011$.
Then, we define $w^j_{i,h}$ by
\begin{eqnarray*}
w^j_{i,0} & = & w_0(\chi_{a_0 a_1 a_2}(x^{3i}_j x^{3i+1}_j x^{3i+2}_j)),\\
w^j_{i,1} & = & w_1(\chi_{a_0 a_1 a_2}(x^{3i}_j x^{3i+1}_j x^{3i+2}_j)).
\end{eqnarray*}
Finally, we define nodes $z_j'$ in the output layer by
$$
z_j' = z_j \oplus ((a_0 \gamma_{n_3}) \lor (a_1 \gamma_{n_3+1}) \lor (a_2 \gamma_{n_3+2})),
$$
which can be realized by adding nodes $\gamma'$, $\gamma''$,
and $z_{j,h}$ ($j=0,\ldots,D-1,h=0,1$) and by
defining the activation functions as:
\begin{eqnarray*}
\gamma' & = & [a_0 \gamma_{n_3} + a_1 \gamma_{n_3+1} + a_2 \gamma_{n_3+2} \geq 1],\\
\gamma'' & = & [\gamma' \geq 1],\\
z_{i,0} & = & [z_i - \gamma'' \geq 1],\\
z_{i,1} & = & [- z_i  + \gamma'' \geq 1],\\
z'_i & = & [z_{i,0} + z_{i,1} \geq 1].
\end{eqnarray*}

For example, when $c_0 c_1 c_2 = 000$, we have the values of
relevant variables as in Table~\ref{tab:c-000}.

\begin{table}[h]
\caption{Values of relevant variables for the case of
$c_0 c_1 c_2 = 000$.}
\label{tab:c-000}
\begin{center}
\begin{tabular}{c||c|c||ccc}
\hline
$x^{3i}_j x^{3i+1}_j x^{3i+2}_j$ & 
$\chi_{0 1 1}(x^{3i}_j x^{3i+1}_j x^{3i+2}_j)$ & 
$z^{3i}_j z^{3i+1}_j z^{3i+2}_j$ &
$(z')^{3i}_j (z')^{3i+1}_j (z')^{3i+2}_j$ \\
\hline
000 & 011 & 111 & 100 \\
001 & 010 & 010 & 001 \\
010 & 001 & 001 & 010 \\
011 & 000 & 000 & 011 \\
100 & 111 & 111 & 100 \\
101 & 110 & 110 & 101 \\
110 & 101 & 101 & 110 \\
111 & 100 & 100 & 111 \\
\hline
\end{tabular}
\end{center}
\end{table}
Indeed, if $x^{3i}_j x^{3i+1}_j x^{3i+2}_j = 010$,
we have $z^{3i}_j z^{3i+1}_j z^{3i+2}_j = 001$ and
$(z')^{3i}_j (z')^{3i+1}_j (z')^{3i+2}_j = 010$ by
the following:
\begin{eqnarray*}
w^j_{i,0} & = & w_0(\chi_{011}(010)) = w_0(001) = 001_{(2)} = 1,\\
w^j_{i,1} & = & w_1(\chi_{011}(010)) = w_1(001) = 100_{(2)} = 4,\\
\mbox{[for $3i$:]} & & \\
y_{j,0} & = & [1-2 \geq 0] = 0,\\
y_{j,1} & = & [4-2 \geq 0] = 1,\\
z_j & = & [0+1 \geq 2] = 0,\\
z_j' & = & 0 \oplus (0 \cdot \gamma_{n_3}) = 0,\\
\mbox{[for $3i+1$:]} & & \\
y_{j,0} & = & [1-4 \geq 0] = 0,\\
y_{j,1} & = & [4-0 \geq 0] = 1,\\
z_j & = & [0+1 \geq 2] = 0,\\
z_j' & = & 0 \oplus (1 \cdot \gamma_{n_3+1}) = 1,\\
\mbox{[for $3i+2$:]} & & \\
y_{j,0} & = & [1-0 \geq 0] = 1,\\
y_{j,1} & = & [4-4 \geq 0] = 1,\\
z_j & = & [1+1 \geq 2] = 1,\\
z_j' & = & 1 \oplus (1 \cdot \gamma_{n_3+2}) = 0.
\end{eqnarray*}

If $n$ can be divided by 3,
the total error will be at most ${\frac {nD}{3 \cdot 2^3}}={\frac {nD}{24}}$.
Otherwise, 
the total error will be at most ${\frac {nD}{3 \cdot 2^3}}={\frac {nD}{24}} + D$
because there might be an additional error per $j$.
Since the maximum width is given by
$\max(\lceil n/3 \rceil + 3,2D+1)$,
we have the following theorem.

\begin{theorem}
\label{thm:approx-b-3}
%% TA21dec>
For any perfect encoder that maps $X_n$
% 21dec >
one-to-one
to $d$-dimensional binary vectors,
with $d=\lceil \log n \rceil$ and $n$ sufficiently large,
there exists a decoder of a constant depth and width
$\max(\lceil n/3 \rceil + 3,2D+1)$ whose average Hamming distance error
is at most $D({\frac {1}{24}}+{\frac 1 n})$.
% For $n \gg (\log n)^k$ for some constant $k$,
% there exists a decoder of a constant depth and width
% $\max(\lceil n/3 \rceil + 3,2D+1)$ whose average Hamming distance error
% is at most $D({\frac {1}{24}}+{\frac 1 n})$ for any set of $n$ input vectors.
%% <TA21dec
\end{theorem}

\subsection{General Case}
\label{sec:approx-general}

As in Section~\ref{sec:perfect-decoder},
we define $n_B=n/B$ and assume w.l.o.g. that $n$ can be divided by $B$.
First, we define $\gamma_i$ nodes ($i=0,\ldots,n_{B}+B-1$) by
\begin{eqnarray*}
\gamma_i & = & [ \lfloor {\frac {Int(\beta)}{B}} \rfloor = i],~~~~
\mbox{for $i=0,\ldots,n_B-1$},\\
\gamma_{n_B+k} & = & [ Int(\beta) = k ~(\mbox{mod}~B)],~~~~
\mbox{for $k=0,\ldots,B-1$}.
\end{eqnarray*}

We construct $y_{j,h}$ nodes ($j=0,\ldots,D-1$, $h=0,\ldots,B-2$) 
and $z_j$ nodes ($j=0,\ldots,D-1$) with the activation functions:
\begin{eqnarray*}
y_{j,h} & = & \left[ \sum_{i=0}^{n_3+B-1} w^j_{i,h} \gamma_i \geq \theta_h \right],\\
z_j & = & \left[ \sum_{h=0}^{B-2} y_{0,h} ~\geq~ B-1 \right].
\end{eqnarray*}

We define $w_{n_B+i,h}$s ($i=0,\ldots.B-1$, $h=0,\ldots,B-2$) by
\begin{eqnarray*}
w_{n_B+i,h} & = & 
\left\{
\begin{array}{ll}
-{\frac {2^B}{4}}, & \mbox{if $i=0$ and $h \in \{0,\ldots,B-2\}$},\\
-{\frac {2^B}{2}}, & \mbox{if $i=h+1$ and $h \in \{0,\ldots,B-2\}$},\\
0, & \mbox{otherwise},
\end{array}
\right.
\end{eqnarray*}

For integer $k \in \{0,\ldots,2^B-1\}$,
let $k_0 k_1 \cdots k_{B-1}$ be the binary representation of $k$
(i.e., $k = Int(k_0 k_1 \cdots k_{B-1})$),
where $k_0$ is the most significant bit.
We define $w_h(k)$ ($h=0,\ldots,B-2$) by
\begin{eqnarray*}
w_0(k_0 k_1 k_2 k_3 \cdots k_{B-1}) & = &
Int(k_{1} k_0 k_2 k_3 \cdots k_{B-1}),\\
w_1(k_0 k_1 k_2 k_3 \cdots k_{B-1}) & = &
Int(k_{2} k_0 k_1 k_3 \cdots k_{B-1}),\\
w_2(k_0 k_1 k_2 k_3 \cdots k_{B-1}) & = &
Int(k_{3} k_0 k_1 k_2 \cdots k_{B-1}),\\
& \cdots &\\
w_{B-2}(k_0 k_1 k_2 k_3 \cdots k_{B-1}) & = &
Int(k_{B-1} k_0 k_1 \cdots k_{B-2}).
\end{eqnarray*}
Then, we define $w^j_{i,h} = w_h(x^{Bi}_j x^{Bi+1}_j \cdots x^{Bi+B-1}_j)$
for $j=0,\ldots,D-1$, $i=0,\ldots,n_B-1$, and $h=0,\ldots,B-2$.
The following proposition is straightforward from
the definitions of $y_{i,h}$s, $w_{n_B+i,h}$s, $w_h(\cdots)$s, and $z_j$s.

\begin{proposition}
If $x^{Bi}_j x^{Bi+1}_j x^{Bi+2}_j \cdots x^{Bi+B-1}_j =
011 \cdots 1$,
then $z^{Bi}_j z^{Bi+1}_j z^{Bi+2}_j \cdots z^{Bi+B-1}_j =
111 \cdots 1$ holds.
Otherwise,
$z^{Bi}_j z^{Bi+1}_j z^{Bi+2}_j \cdots z^{Bi+B-1}_j = 
x^{Bi}_j x^{Bi+1}_j x^{Bi+2}_j \cdots x^{Bi+B-1}_j$
 holds.
\end{proposition}

As in Section~\ref{sec:approx-b-3},
it is seen from this proposition that if all bit
patterns are distributed uniformly at random,
the expected Hamming distance between $\xvec^i$ and $\yvec^i$
is ${\frac {D}{B 2^B}}$.
We can modify the network so that 
the average Hamming distance between  $\xvec^i$ and $\yvec^i$
is always no more than ${\frac {D}{B 2^B}}$ as well.
The modification method is given below, which is almost the same as that in
Section~\ref{sec:approx-b-3}.

Let 
$
\#_{b_0 b_1 \cdots b_{B-1}}
= \sum_{j=0}^{D-1} |\{ i | x^{Bi}_j x^{Bi+1}_j \cdots x^{Bi+B-1}_j = b_0 b_1 \cdots b_{B-1} \}|$
where $b_0 b_1 \cdots b_{B-1} \in \{0,1\}^B $.
Let $c_0 c_1 \cdots c_{B-1} =
\mbox{argmin}_{b_0 b_1 \cdots b_{B-1}} \#_{b_0 b_1 \cdots b_{B-1}}$,
where the tie can be broken in any way.
Then, we let $a_0 a_1 \cdots a_{B-1} =
(0 \oplus c_0) (1 \oplus c_1) \cdots (1 \oplus c_{B-1})$.
For example. $a_0 a_1 a_2 a_3 = 0110$ when $c_0 c_1 c_2 c_3 = 0001$.

Here we define $\chi_{a_0 a_1 \cdots a_{B-1}}(x_0 x_1 \cdots x_{B-1})$ by
$$
\chi_{a_0 a_1 \cdots a_{B-1}}(x_0 x_1 \cdots x_{B-1}) = (x_0 \oplus a_0) (x_1 \oplus a_1) \cdots (x_{B-1} \oplus a_{B-1}).
$$
Then, we define $w^j_{i,h}$ by
\begin{eqnarray*}
w^j_{i,h} & = & w_h(\chi_{a_0 a_1 \cdots a_{B-1}}(x^{Bi}_j x^{Bi+1}_j \cdots x^{Bi+B-1}_j)),
\end{eqnarray*}
for $j=0,\ldots,D-1$, $i = 0,\ldots,n_B-1$, and $h=0,\ldots,B-2$.
Finally, we define nodes $z_j'$ in the output layer by
$$
z_j' = z_j \oplus ( \bigvee_{h=0}^{B-1} a_h \gamma_{n_B+h}),
$$
which can be realized by adding nodes $\gamma'$, $\gamma''$,
and $z_{j,h}$ ($j=0,\ldots,D-1,h=0,1$) and defining the activation functions by:
\begin{eqnarray*}
\gamma' & = & [\sum_{h=0}^{B-1} a_h \gamma_{n_B+h} ~\geq~ 1],\\
\gamma'' & = & [\gamma' \geq 1],\\
z_{i,0} & = & [z_i - \gamma'' \geq 1],\\
z_{i,1} & = & [- z_i  + \gamma'' \geq 1],\\
z'_i & = & [z_{i,0} + z_{i,1} \geq 1].
\end{eqnarray*}

Since the maximum width is given by
$\max(\lceil n/B \rceil + B,(B-1)D+1)$,
we have the following theorem.

\begin{theorem}
\label{thm:approx-general}
%% TA21dec>
For any perfect encoder that maps $X_n$
% 21dec >
one-to-one
to $d$-dimensional binary vectors,
with $d=\lceil \log n \rceil$ and sufficiently large $n$,
there exists a decoder of a constant depth and width
$\max(\lceil n/B \rceil + B,(B-1)D+1)$ whose average Hamming distance error
is at most $D({\frac {1}{B 2^B}}+{\frac 1 n})$,
where $B$ is any integer such that $3 \leq B \leq n$.
% For $n \gg (\log n)^k$ for some constant $k$,
% there exists a decoder of a constant depth and width
% \max(\lceil n/B \rceil + B,(B-1)D+1)$ whose average Hamming distance error
%vis at most $D({\frac {1}{B 2^B}}+{\frac 1 n})$ for any set of $n$ input vectors,
% where $B$ is any integer such that $3 \leq B \leq n$.
%% <TA21dec
\end{theorem}

Note that this theorem is a generalized version of
Theorem~\ref{thm:approx-b-3} and
the resulting width is smaller than that of Theorem~\ref{thm:perfect-decoder}
for all $B \geq 3$ and $\lceil n/B \rceil + B > BD$
(i.e., $D < n'$ for $n' \approx n/B^2$)
although the corresponding decoder is an approximate one and uses more layers.
Note also that if $B = 2^K$, the architecture of the decoder 
can be explicitly described as $d/\lceil n/B \rceil + B/(B-1)D+1/D+1/2D/D$
where $d = \lceil \log n \rceil$.
Furthermore, if we can consider the average case error
over binary input vectors given uniformly at random,
this architecture can be reduced to $d/\lceil n/B \rceil + B/(B-1)D/D$.

\section{Concluding Remarks}

In this paper, we showed an improved upper bound on the size/width
of the decoder in a Boolean threshold autoencoder.
We also showed a lower bound on the size of the decoder.
Although these bounds are relatively close, there still exists a gap.
Therefore, closing the gap is left as an open problem as well as
giving a lower bound on the size/width of the encoder part.
We also
%>AM
%showed decoders allowing small errors.
%However, the width of these decoders is not significantly smaller than
%that of exact decoders.
constructed decoders that are permitted to make small errors.
However, the widths of these decoders are not significantly smaller than
those of exact decoders.
%<AM
Therefore, improvement of decoders allowing errors is left as future work.

%
%
%
%\begin{thebibliography}{99}

\bibliographystyle{plain}

\bibliography{autowidtharxiv}

%\end{thebibliography}

\end{document}